\documentclass{article}
\PassOptionsToPackage{numbers, compress}{natbib}
\usepackage[preprint]{neurips_2024}

\usepackage[utf8]{inputenc} % allow utf-8 input
\usepackage[T1]{fontenc}    % use 8-bit T1 fonts
\usepackage[hidelinks]{hyperref}       % hyperlinks
\usepackage{url}            % simple URL typesetting
\usepackage{booktabs}       % professional-quality tables
\usepackage{amsfonts}       % blackboard math symbols
\usepackage{nicefrac}       % compact symbols for 1/2, etc.
\usepackage{microtype}      % microtypography
\usepackage{xcolor}         % colors
\usepackage{makecell}
\usepackage{graphicx}
\usepackage{caption}
\usepackage{multirow}
\usepackage{multicol}
\usepackage{enumitem}
\usepackage{subfigure}
\usepackage{amsmath}
\usepackage{amsthm}
\usepackage{wrapfig}
\usepackage{bbm}
\usepackage{bm}
\usepackage{threeparttablex}
\usepackage{xspace}
\usepackage{algpseudocode}
\usepackage[ruled,vlined]{algorithm2e}
\usepackage{bbding}
\usepackage{wrapfig}
\usepackage{adjustbox}
\usepackage{pifont}

\AtBeginDocument{%
  \providecommand\BibTeX{{%
    \normalfont B\kern-0.5em{\scshape i\kern-0.25em b}\kern-0.8em\TeX}}}

\newcommand{\our}{\textsc{OrthoFL}\xspace}

\makeatletter 
  \newcommand\figcaption{\def\@captype{figure}\caption} 
  \newcommand\tabcaption{\def\@captype{table}\caption} 
\makeatother
\newtheorem{theorem}{Theorem}[section]

\newtheorem{definition}{Definition}
\newtheorem{lemma}[theorem]{Lemma}
\newcommand{\smallsection}[1]{\noindent\textbf{#1}.}

\title{Orthogonal Calibration for \\Asynchronous Federated Learning}

\author{%
  Jiayun Zhang \\
  UC San Diego\\
  \texttt{jiz069@ucsd.edu} \\
  \And
  Shuheng Li \\
  UC San Diego\\
  \texttt{shl060@ucsd.edu} \\
  \And
  Haiyu Huang \\
  UC Los Angeles\\
  \texttt{haiyu@g.ucla.edu} \\
  \And
  Xiaofan Yu \\
  UC San Diego\\
  \texttt{x1yu@ucsd.edu} \\
  \And
  Rajesh K. Gupta \\
  UC San Diego\\
  \texttt{rgupta@ucsd.edu} \\
  \And
  Jingbo Shang \\
  UC San Diego \\
  \texttt{jshang@ucsd.edu}
}

\begin{document}

\maketitle

\begin{abstract}
    Asynchronous federated learning mitigates the inefficiency of conventional synchronous aggregation by integrating updates as they arrive and adjusting their influence based on staleness. Due to asynchrony and data heterogeneity, learning objectives at the global and local levels are inherently inconsistent---global optimization trajectories may conflict with ongoing local updates. Existing asynchronous methods simply distribute the latest global weights to clients, which can overwrite local progress and cause model drift.
In this paper, we propose \our, an orthogonal calibration framework that decouples global and local learning progress and adjusts global shifts to minimize interference before merging them into local models. In \our, clients and the server maintain separate model weights. Upon receiving an update, the server aggregates it into the global weights via a moving average. For client weights, the server computes the global weight shift accumulated during the client's delay and removes the components aligned with the direction of the received update. The resulting parameters lie in a subspace orthogonal to the client update and preserve the maximal information from the global progress. The calibrated global shift is then merged into the client weights for further training. Extensive experiments show that \our improves accuracy by \textbf{9.6\%} and achieves a \textbf{12$\times$} speedup compared to synchronous methods. Moreover, it consistently outperforms state-of-the-art asynchronous baselines under various delay patterns and heterogeneity scenarios.
\end{abstract}

\section{Introduction}
Federated learning~\cite{mcmahan2017communication} is a distributed learning paradigm that allows multiple parties to collaboratively train models without sharing data. The most widely adopted federated learning protocols~\cite{mcmahan2017communication,li2020federated,karimireddy2020scaffold,li2021model,wang2020tackling,zhang2023navigating,zhang2024few} follow a \textit{synchronous} update procedure, where the server waits for all selected clients to finish local training before aggregating updates. This synchronization becomes inefficient under heterogeneous resource conditions, where clients differ in compute, network bandwidth, and data volume, due to distinct device configurations and user-system interaction patterns. 

Asynchronous federated learning offers an alternative approach that aggregates client updates as they arrive (see Figure~\ref{fig:intro-asynchrony}), reducing idle time caused by slower clients. In this setting, when a client is performing local training, the server continuously aggregates updates from other clients, shifting the global model to new states. By the time the client's update reaches the server, it may be stale. 
\begin{wrapfigure}{r}{0.55\textwidth}
\centering
    \includegraphics[width=\linewidth]{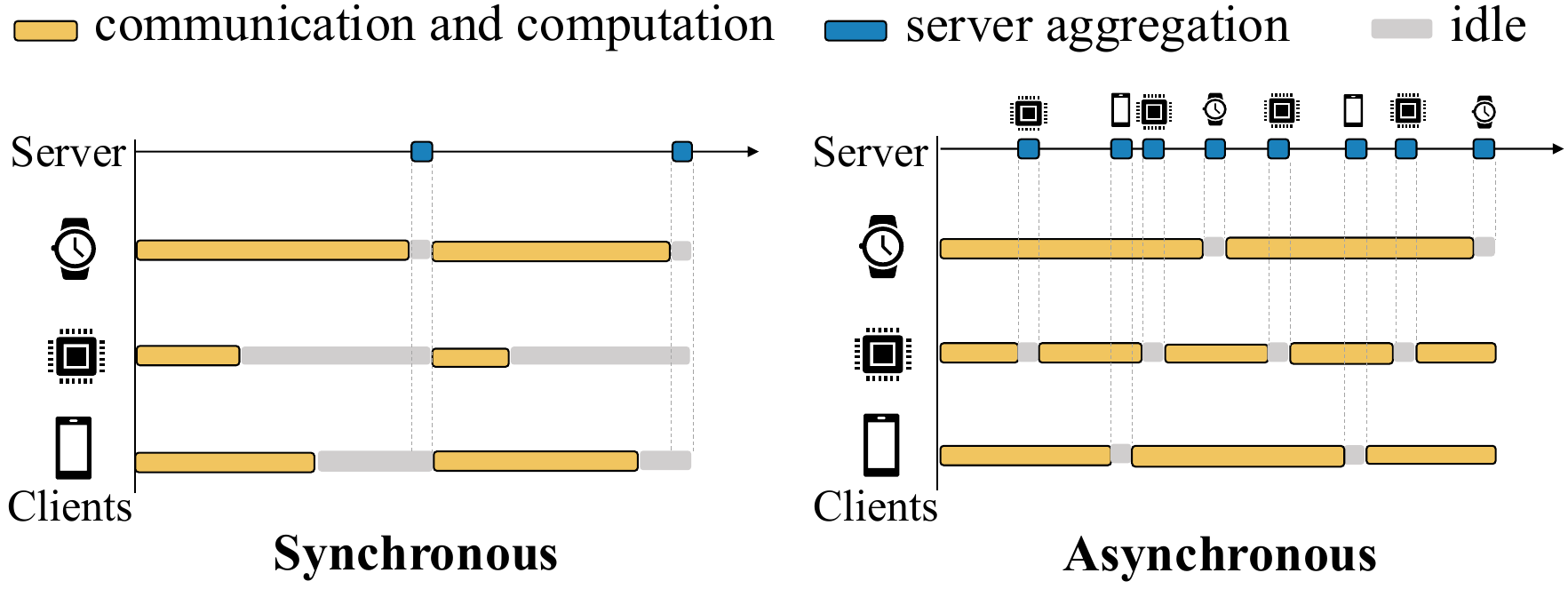}
    \caption{Time synchrony in federated learning. Asynchronous methods reduce idle time and improve resource utilization, suited for large-scale deployments.}
\label{fig:intro-asynchrony}
\end{wrapfigure}
Existing methods manage such staleness by applying a decay factor to updates before aggregating them into the global model~\cite{xie2019asynchronous,liu2024fedasmu,zang2024efficient,su2022asynchronous}. The updated global model is directly returned to the clients for further training. Although down-weighting a stale update reduces its negative impact on global progress, it also diminishes the integration of meaningful knowledge from the client. Moreover, due to data heterogeneity, the optimization objectives of the global and client models are inherently inconsistent---while the global model aims to optimize for the overall data distribution, individual clients minimize loss on their local data. Distributing the latest global parameters to clients for subsequent training can introduce conflicts with their local optimization steps, potentially reversing local gains and leading to oscillations in training.

To address the challenge, we propose to decouple global and local learning progress and calibrate weight shifts to reduce interference during client weight merging. The key insight is that, in the high-dimensional parameter space of neural networks, there are multiple viable directions for effective optimization~\cite{wortsman2021learning}. Some of these directions severely disrupt performance on previously learned distributions, while others have little impact. This offers an opportunity to avoid disruptive components in asynchronous updates and preserve both global progress and client-specific contributions.

We introduce \our, orthogonal calibration for asynchronous federated learning. Our design is motivated by two goals: (1) minimizing interference between global and local optimization by sharing global information perpendicular to client updates, and (2) selecting the most informative direction within the orthogonal hyperplane to maximize knowledge sharing. Specifically, \our maintains separate global and client model weights to accommodate their distinct optimization objectives. When the server receives a client update, the global weights are updated via a moving average with an adaptive decay factor accounting for staleness. For the client model, \our identifies the global weight shift (induced by other clients) during the client's delay. It projects this shift onto the direction of the received client update and subtracts this projected component. The remaining parameters lie in a subspace orthogonal to the client update. Through analysis, we show that this orthogonal calibration strategy keeps maximal global progress while minimizing interference with the local update. The calibrated global shift is then merged with the client model for further training. 

For evaluation, we incorporate realistic delay distributions to reflect the heterogeneous nature of real-world deployments. \our demonstrates an average of 9.6\% accuracy improvement across datasets from diverse application scenarios compared to synchronous methods and a 12$\times$ speedup in reaching a target accuracy, Moreover, it outperforms state-of-the-art asynchronous baselines. We also explore various simulated delay distributions and data heterogeneity levels to understand their impact on model performance and convergence speed.
In summary, our contributions are as follows:
\begin{itemize}[leftmargin=*]
    \item We analyze the key challenges of asynchronous federated learning---the inconsistency of global and local objectives and the detrimental effect of stale updates in heterogeneous environments.
    \item We propose a novel orthogonal calibration method that maintains separate global and local model weights. It projects global shifts onto orthogonal subspaces of local updates before sharing them with clients. This approach reduces interference, preserves meaningful contributions from both global progress and local updates, and enhances knowledge sharing. 
    \item We demonstrate the effectiveness and robustness of \our through comprehensive experiments on multiple datasets and various delay scenarios, providing insights on practical design considerations for large-scale federated learning systems.
\end{itemize}

\section{Preliminaries} \label{sec:preliminary}

\subsection{Asynchronous System Architecture}
In an asynchronous federated learning setup, a central server coordinates the training of a global model $W$ using data distributed across $M$ clients. Each client $m \in \{1, 2, \ldots, M\}$ possesses its own local dataset $\mathbf{D}_m$. The data distribution of client \(m\) is denoted as \(\mathcal{P}_m\).
The objective is to train a global model $W$ that generalizes well across the combined data distribution of all clients.
Formally, we aim to solve the following optimization problem:
\begin{align}
W^* = \arg \min_{W} \frac{1}{M} \sum_{m=1}^{M} \mathbb{E}_{(x,y)\sim \mathcal{P}_m} \ell \left( f(x; W), y \right),
\end{align}
where $W$ denotes the global weights, \(\ell\) the loss function, and $f(x; W)$ the prediction of the model on data $x$ with model weights $W$.

Clients perform local training and communicate their updates to the central server at different times.
Let $T$ be the number of global rounds. For $t \in \{1, \dots, T\}$, denote $m_t \in \{1, \dots, M\}$ as the client that communicates with the server at the $t$-th round, and $\tau_t$ as the round when client $m_t$ last communicated with the server. We define the staleness of the client update as follows: 
\begin{definition}[Staleness] Staleness quantifies the delay between a client's updates, representing the number of global rounds since the client last communicated with the server. Formally, let $t$ be the current global round, and $\tau_t$ the global round when the server last received updates from client $m_t$. The staleness of client $m_t$ is defined as $t - \tau_t$, where $t - \tau_t \geq 1$. A staleness of 1 indicates no delay. 
\end{definition}
For simplicity, we will drop the subscripts on $m_t$ and $\tau_t$ with no ambiguity from now on.

\subsection{A Motivating Study}
We conduct an experiment on MNIST~\cite{deng2012mnist} with a LeNet5~\cite{lecun1998gradient} model to analyze the challenges in asynchronous federated learning.
We simulate the scenario with two clients: one with a 10-second latency and the other with 30, 60, or 100 seconds. We adopt the asynchronous method, FedAsync~\cite{xie2019asynchronous}, where client updates are aggregated with decay factors based on latency. Let $W^{(t)}$ denote the global weights at $t$-th round before aggregation, and $W^{(t_+)}$ the global weights after aggregation. Similarly, let $W_m^{(t)}$ represent the model weights of client $m$ at $t$-th round. The aggregation follows:
\begin{align}
\beta_t &= (t-\tau)^{-a}\cdot\beta, \\
W^{(t_+)} &= (1-\beta_t) W^{(t)} + \beta_t W^{(t)}_m
\label{eq:fedasync}
\end{align}
where $\beta$ and $a$ are hyperparameters set to $\beta=0.6$ and $a=0.5$ as reported in FedAsync.
To create non-IID data, each client is assigned a non-overlapping half of the MNIST classes. 

\begin{wrapfigure}{r}{0.55\textwidth}
    \centering
    \subfigure[Global Accuracy]{
    \centering
    \includegraphics[width=0.47\linewidth]{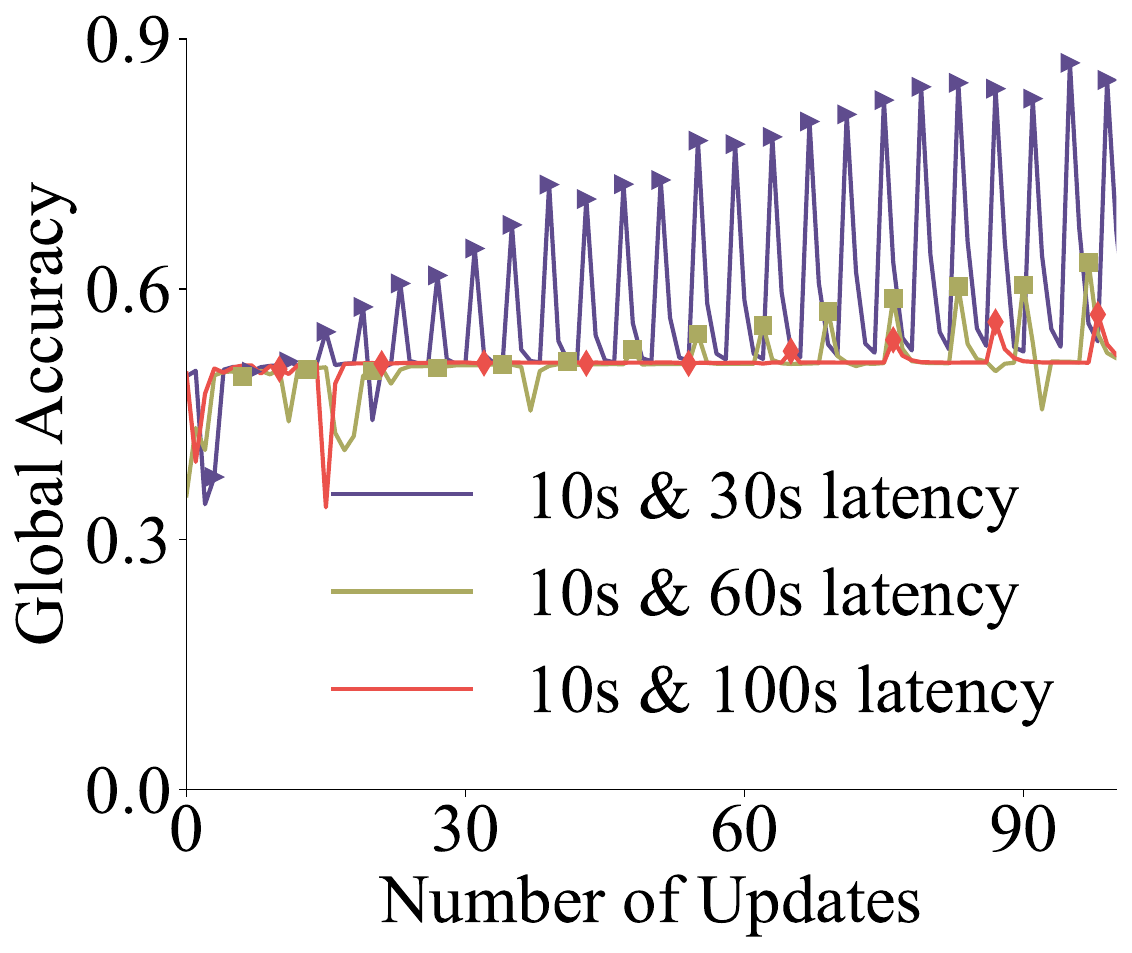}
    \label{fig:motivation-noniid-mnist-acc}
    }
    \subfigure[Global Weight Change]{
    \centering
    \includegraphics[width=0.47\linewidth]{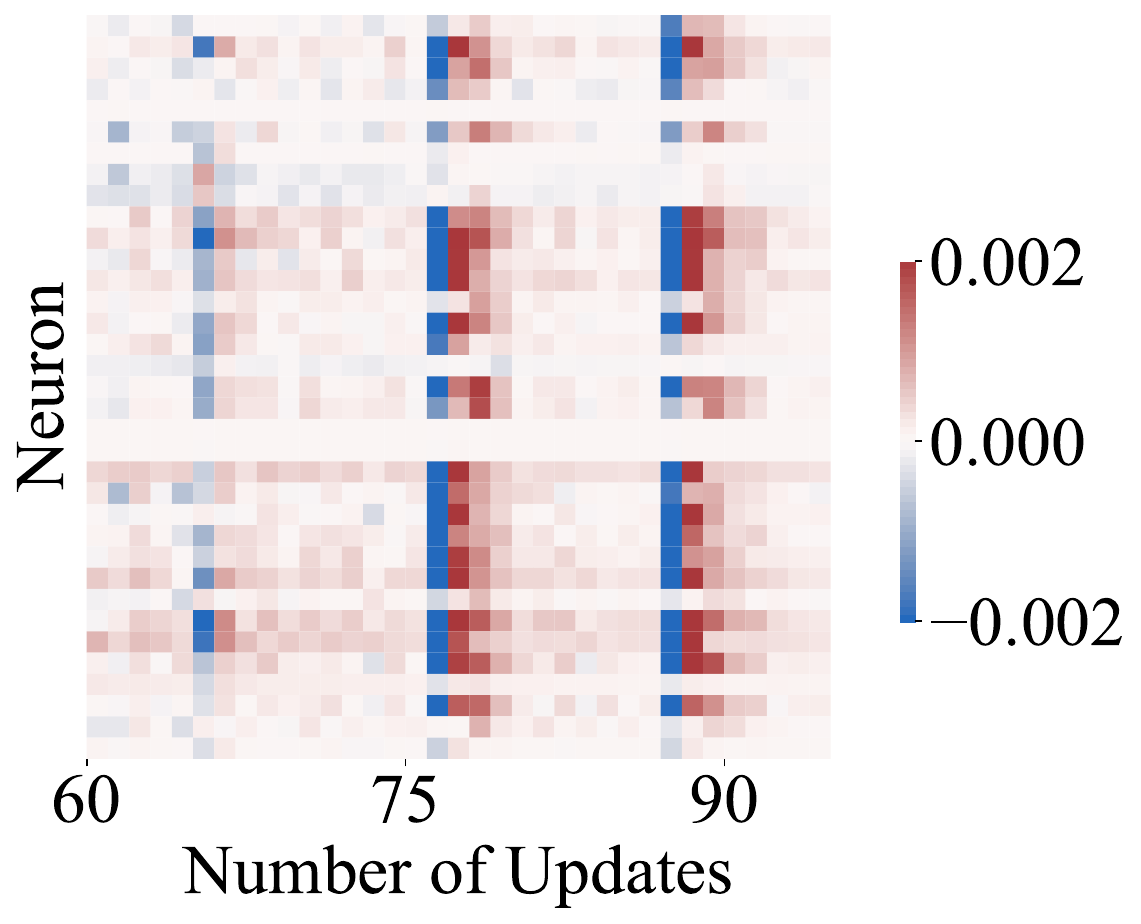}
    \label{fig:motivation-noniid-mnist-grad}
    }
    \caption{
    Asynchronous learning with a fast client (10s latency) and a slow client (30/60/100s) assigned non-overlapping classes. Due to objective inconsistency: (a) accuracy spikes when the slow client updates, followed by drops as the fast client updates; (b) update directions shift abruptly when the active client switches.
    }
\end{wrapfigure}

The global performance is shown in Figure~\ref{fig:motivation-noniid-mnist-acc}, where markers represent updates from the ``slow'' client with longer latency. We observe an increase in accuracy when the server aggregates updates from the slow client, as these updates introduce knowledge of previously undertrained classes. However, this gain is gradually lost, with accuracy declining to around 0.5 after several updates from the faster client. 
This suggests the fast client's updates override the contributions of the slower client. 
Moreover, as the latency of the slower client increases, the decay factor $\beta_t$ for integrating its updates decreases. This weakens its contribution to the global model and slows convergence, especially under non-IID data, as valuable knowledge from the slower client is not fully utilized.

Figure~\ref{fig:motivation-noniid-mnist-grad} visualizes changes in global model weights in the final hidden layer before the classifier in the case where the latency of the two clients is 10 and 100 seconds respectively.
The y-axis represents neurons, and the x-axis represents the number of updates.
The color indicates the direction and degree of global weight changes, with red representing an increase and blue a decrease. 
We observe abrupt shifts occur when switching between clients. Updates from the slow client often decrease the neuron weights (blue), while subsequent updates from the fast client increase the weight values (red), pulling the model in opposite directions. 
The antagonistic behavior is due to objective inconsistency---while the global model optimizes for the overall distribution, client updates follow distinct local objectives, driving oscillations in weight aggregation.
\section{Method: Orthogonal Weight Calibration}
\subsection{\our Algorithm}
\label{sec:method}

Once receiving a client update, \our immediately integrates it into the global model. 
\our maintains \emph{separate} variables for global and client models. 
Before merging global weight shift to the client model, the server orthogonalizes the global shift against the received client update. This orthogonality allows the client to incorporate global progress while continuing its local optimization without disruption. The pseudo-code is presented in Appendix~\ref{sec:pseudo-code}. 

\smallsection{Global Aggregation via Moving Average} 
Denote $W^{(t)}$ as the global model weights at the $t$-th round before client update and $W^{(t_+)}$ after update. Similarly, let $W_m^{(t)}$ be the client $m_t$'s local model weights at the $t$-th round before update and $W_m^{(t_+)}$ after update. Note that $W^{(t+1)} := W^{(t_+)}$ as the global model weights stay unchanged after communication with a client before the next client update. We update the global model with a moving average:
\begin{equation}
    W^{(t_+)} = (1-\beta_t) W^{(t)} + \beta_t W_m^{(t)},
\label{eq:aggregation}
\end{equation}
where $\beta_t$ controls the contribution of client $m$'s current update to the aggregation. We let $\beta_t := s_a(t-\tau) \cdot \beta$ with $\beta \in (0,1)$ and $s_a(x) = x^{-a}$ for some $a > 0$ so that update with a larger staleness has a smaller contribution, and thereby decreasing the influence of client update with long delay.

\smallsection{Calibration on Client Updates} 
To minimize interference caused by asynchronous updates, we orthogonalize the global weight change that occurs between the client's successive updates against client's local udpate before sending it to the client for the next round of training.
Formally, when the server receives an update from client $m$, if the staleness $t - \tau > 1$, it calculates the local weight change from its last update to its current update:
\begin{equation}
  \Delta W_m = W_m^{(t)} - W_m^{(\tau_+)} .
\end{equation}

Similarly, the server calculates the global weight shift due to aggregating updates from other clients during this period:
\begin{equation}
\Delta W = W^{(t)} - W^{(\tau_+)}.
\end{equation}

To update client $m$ with global progress, the server computes the orthogonal component of $\Delta W$ with respect to $\Delta W_m$. The orthogonalization is done for the weight of each layer through removing the component that is parallel to $\Delta W_m$. Let $\Delta W^{l}$ and $\Delta W^{l}_m$ be the change in the layer $l$ of the global weights and the local weights respectively. The component of $\Delta W^l$ orthogonal to $\Delta W_m^l$ is:
\begin{equation}
\label{eq:8}
\begin{aligned}
    \Delta W^{l\perp} = \Delta W^{l} - \text{proj}_{\Delta W_m^{l}}(\Delta W^{l})
    = \Delta W^{l} - \frac{\Delta W^{l} \cdot \Delta W_m^{l}}{\Delta W_m^{l} \cdot \Delta W_m^{l}} \Delta W_m^{l}.
\end{aligned} 
\end{equation}

Let $\Delta W^{\perp}$ represent the aggregation of $\Delta W^{l\perp}$ across all layers.
This orthogonal component $\Delta W^{\perp}$ ensures that the updates from the other clients during delay $t-\tau$ do not interfere with the client's local progress, as it removes any component of the global weight change during delay along the direction of the local update.
$\Delta W^{\perp}$ is then sent to current client $m$ to form the new client model weight for the subsequent round of local training on  client $m$:
\begin{equation}
    W_m^{(t_+)} = W_m^{(t)} + \Delta W^\perp.
\end{equation}

\subsection{Mathematical Basis of Orthogonalization}
\label{sec:math}
A gradient update orthogonal to previously accumulated gradients helps preserve existing model behavior and minimizes unintended changes to its outputs~\cite{farajtabar2020orthogonal}.
Due to the high-dimensional parameter space of neural networks, there are multiple directions that are orthogonal to the stale local weight update. We have the following lemma showing that the orthogonalization strategy in \our preserves the maximal information from the global weight shift vectors perpendicular to the local update direction. The proof is given in Appendix~\ref{sec:proof}.

For $v, w \in \mathbb R^d$, let $\langle v, w \rangle := \sum_{i=1}^d v_i w_i$ be the standard inner product on $\mathbb R^d$.
\begin{lemma}
\label{eq:lemma}
    Let $v \in \mathbb R^d$ and $\mathcal{U} = \{u_1, \cdots, u_k\}$ be an orthonormal set for some $k < d$. Then for any $w \in (\mathrm{span}\,\mathcal{U})^\perp$, 
    \begin{equation}
        \label{eq:4}
        \Vert v - v^\perp \Vert \le \Vert v - w \Vert,
    \end{equation}
    where $v^\perp:= v - \sum_{i=1}^k \langle v, u_i\rangle u_i$ denote the component of $v$ orthogonal to $\mathcal{U}$. Moreover, the angle between $v$ and $v^\perp$ is less than any angle between $v$ and $w$ for $w \in (\mathrm{span}\,\mathcal{U})^\perp$.
\end{lemma}

We apply this lemma in our case when $v = \Delta W^l$ and $\mathcal{U} = \{\Delta W_m^l\}$. It implies $\Delta W^{l\perp}$ in equation~\eqref{eq:8} is the unique vector among those perpendicular to $\Delta W_m^l$ with the smallest magnitude of $\Vert \Delta W^l - \Delta W^{l\perp}\Vert$ and the smallest angle with $\Delta W^l$. This allows the server to pass the maximum knowledge from global progress to client $m$ without interfering with its most recent update.

\begin{figure}[t]
\centering
    \includegraphics[width=\linewidth]{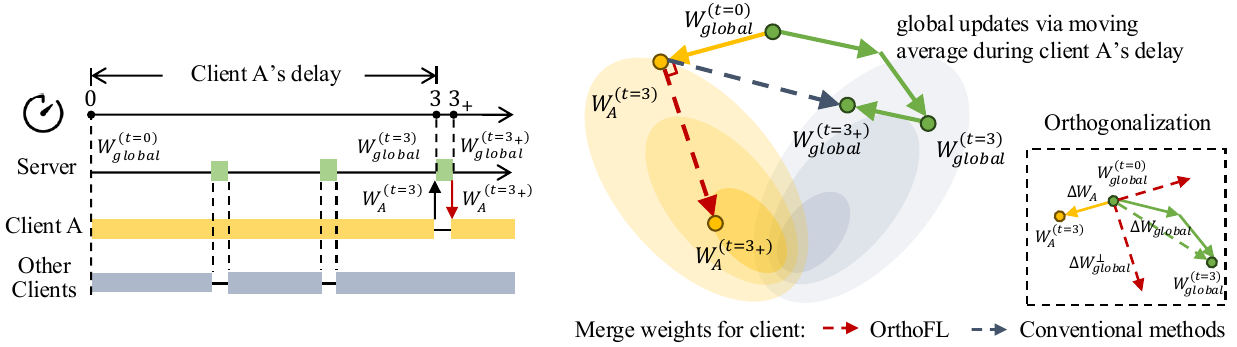}
    \caption{An example of optimization trajectories. Shaded regions represent iso-loss contours for client A (yellow) and other clients (gray). Deeper colors indicate lower loss areas. \our removes conflicting components via orthogonalization, merging updates with minimal interference.
    }
\label{fig:method-orthogonal-calibration}
\end{figure}

\subsection{Visualization of Optimization Trajectories}
We illustrate the advantage of our method by visualizing an example of optimization trajectories. As shown in Figure~\ref{fig:method-orthogonal-calibration}, client A begins local training from the global state $W^{(t=0)}_\text{global}$. Before A's update arrives, the server aggregates two updates from other clients into the global model. Consequently, just before aggregating A's update at $t=3$, the global model has evolved to $W^{(t=3)}_\text{global}$. Meanwhile, client A finishes local training and submits the updated parameters $W^{(t=3)}_A$. The shaded regions show the iso-loss contours for client A (yellow) and the collective optimization space of other clients (gray).

The global weight is updated via a moving average and becomes $W^{(t=3_+)}_\text{global}$. In conventional asynchronous methods, this global weight is directly assigned to client A (gray dashed line). This would push the model farther from A's optimization objective than $W^{(t=3)}_A$, reversing A's learning progress. \our mitigates this by removing the component of the global weight shift that is parallel to $\Delta W_A$. This ensures that the calibrated parameters are orthogonal to A's update direction. Finally, the calibrated global shift is merged into A's model (red dashed line), which becomes $W^{(t=3+)}_A$. This way, \our reduces interference due to staleness and objective inconsistencies while preserving meaningful contributions at both global and local levels. 
\section{Experiments}\label{sec:experiment}
\subsection{Experiment Setup}
\smallsection{Compared Methods}
We consider baselines including synchronous methods, FedAvg~\cite{mcmahan2017communication}, FedProx~\cite{li2020federated}, FedAdam~\cite{reddi2020adaptive}, semi-asynchrounous methods, FedBuff~\cite{nguyen2022federated} and CA$^2$FL~\cite{wang2024tackling}, and fully-asynchrounous methods, FedAsync~\cite{xie2019asynchronous}. The detailed descriptions of these compared methods are provided in Appendix~\ref{sec:baseline-methods}.

\begin{wrapfigure}{r}{0.5\textwidth}
\renewcommand\tabcolsep{3pt}
    \centering
    \captionsetup{type=table}
    \caption{Datasets and models in the experiments.}
    \small
    \scalebox{0.81}{
    \begin{tabular}{lccccc}
    \toprule
    Datasets & Clients & Avg. $\lvert \mathbf{D}_m\rvert$ & Model & Data Type \\
    \midrule
    CIFAR-10          & 10 & 4000 & VGG11        & image\\
    MNIST            & 10 & 6000 & LeNet5        & image\\
    20 Newsgroups    & 20 & 566 & DistilBERT     & text\\
    HAR              & 21 & 350 & ResNet18       & time-series\\
    CIFAR-100         & 100 & 400 & MobileNetV2  & image\\
    \bottomrule
    \end{tabular}
    }
    \label{tab:exp-dataset-and-model}
\end{wrapfigure}
\smallsection{Applications and Datasets}
Table~\ref{tab:exp-dataset-and-model} summarizes the setups for each dataset. 
We conduct experiments on five datasets, CIFAR-10 and CIFAR-100~\cite{krizhevsky2009learning}, MNIST~\cite{deng2012mnist}, 20 Newsgroups~\cite{lang1995newsweeder}, and HAR~\cite{anguita2013public}, including three distinct data types: image, text, and time series.
To assess the robustness of \our across different model architectures, we pair each dataset with a model suited to its data type. To evaluate \our's performance in parameter-efficient fine-tuning (PEFT) settings~\cite{han2024parameter}, we employ a pretrained DistilBERT~\cite{sanh2019distilbert} from Hugging Face\footnote{\url{https://huggingface.co/distilbert/distilbert-base-uncased}} for evaluations on the 20 Newsgroups dataset and fine-tune it using Low-Rank Adaptation (LoRA)~\cite{hu2021lora}, which reduces the number of trainable parameters. More details are given in Appendix~\ref{sec:application-and-dataset}.

\smallsection{Data Heterogeneity} For the HAR dataset, clients are naturally divided based on the individual subjects, as each subject represents a client. For the other datasets, we set the number of clients equal to the number of classes in each dataset. To create non-IID client distributions, we follow prior work~\cite{hsu2019measuring} and use a Dirichlet distribution $Dir(\alpha=0.1)$ to derive class distribution.

\smallsection{Delay Simulation}
To ensure controlled evaluation and avoid variability in federated learning deployments, we simulate client delays using measurements (including communication and computational latency) from prior work~\cite{yu2023async}, which were collected using Raspberry Pi (RPi) devices in different home environments with wireless connectivity. To account for differences in the size of models and datasets, we use an RPi 4B which has comparable computational capabilities as the reported ones to measure latency for training one round under each model and dataset configuration. Each configuration is tested five times to derive the average training time. The computational latency for datasets in our experiment is then adjusted based on the ratio between our measured time and the computation latency in~\cite{yu2023async}. Similarly, communication latency was scaled based on the model size in bytes compared to the model used in the original delay collection.
Using these measurements, we calculate the mean and variance of latency for every device. 
During simulations, we assign the statistics to clients randomly and sample latency for each round from a Gaussian distribution parameterized by the assigned statistics.
For fair comparisons, the order of client updates is kept consistent with a fixed set of random seeds across compared methods.
Besides these real-world measurements, we also investigate model performance under additional delay distributions in Section~\ref{sec:exp-exploratory}.

\smallsection{Evaluation Metrics}
We report accuracy after training for sufficient clock cycles to ensure the method reaches stable performance. For a fair comparison, we fix the same training time across all methods. In addition, we compare the time spent in reaching a target accuracy---set as the 95\% of the lowest final accuracy among all compared methods. FedAvg serves as the baseline of time consumption (i.e., 1$\times$), and we report the relative time for other methods.

\subsection{Main Experiment Results}
Table~\ref{tab:exp-main} summarizes the results. \our consistently outperforms the compared methods across all datasets—it not only converges faster but also achieves higher accuracy. Notably, the advantage becomes more obvious in the setting with a larger number of clients, as seen in CIFAR-100. This is because, for synchronous methods, the client sampling rate decreases as the number of clients increases (e.g., only 10 out of 100 clients are sampled per round), leading to longer wait times and slower convergence. Similarly, for baseline asynchronous methods, although their aggregation mechanisms are designed to mitigate the influence of model staleness, they fail to effectively address the challenges posed by data heterogeneity. In contrast, the calibration mechanism in \our alleviates the negative impact of both stale model updates and the model divergence caused by data heterogeneity, ensuring faster convergence and improved performance.

Furthermore, as shown in Figure~\ref{fig:exp-main}, \our demonstrates reduced fluctuations in accuracy over training time. These fluctuations, observed in compared methods, are caused by conflicting updates from clients with non-identical data distributions.

\begin{table*}[t]
  \centering
  \caption{Main results (\%) including average accuracy, standard deviation, and time relative to FedAvg. \our reaches the target accuracy more quickly and achieves higher accuracy.}
  \scalebox{0.77}{
    \begin{tabular}{lcccccccccc}
    \toprule
    \multicolumn{1}{l}{\multirow{2}[2]{*}{Methods}} & \multicolumn{2}{c}{MNIST} & \multicolumn{2}{c}{CIFAR-10} & \multicolumn{2}{c}{20 Newsgroups} & \multicolumn{2}{c}{HAR} & \multicolumn{2}{c}{CIFAR-100} \\
    \cmidrule(lr){2-3}\cmidrule(lr){4-5}\cmidrule(lr){6-7}\cmidrule(lr){8-9}\cmidrule(lr){10-11} & Accuracy   & Time  & Accuracy   & Time  & Accuracy & Time  & Accuracy & Time & Accuracy & Time\\
    \midrule
    FedAvg~\cite{mcmahan2017communication} & 92.2±1.1 & 1$\times$ & 73.8±2.4 & 1$\times$     & 58.1±3.1 & 1$\times$ & 84.2±0.6 & 1$\times$ & 22.2±0.4 & 1$\times$ \\
    FedProx~\cite{li2020federated} & 90.5±0.9 & 1.09$\times$ & 73.3±2.2 & 0.90$\times$ & 58.6±3.1 & 0.98$\times$ & 84.0±0.9 & 0.98$\times$ & 20.4±0.6 & 1.18$\times$\\
    FedAdam~\cite{reddi2020adaptive} & 93.8±2.3 & 0.91$\times$ & \underline{74.6±1.5} & 0.78$\times$ & 58.9±1.8 & 1.08$\times$ & 87.1±3.3 & 0.68$\times$ & 42.1±1.0 & 0.36$\times$\\ \midrule
    FedAsync~\cite{xie2019asynchronous} & 95.4±0.7 & 0.39$\times$ & 74.3±1.1 & \underline{0.48$\times$} & 61.8±3.8 & 0.27$\times$ & 87.6±1.9 & \underline{0.26$\times$} & 47.9±0.9  & 0.20$\times$\\
    FedBuff~\cite{nguyen2022federated} & 93.6±2.1 & 0.46$\times$ & 73.6±3.2 & 0.55$\times$ & 62.1±2.0 & 0.34$\times$ & 87.4±1.0 & 0.30$\times$ & \underline{62.3±0.5} & 0.12$\times$\\
    CA$^2$FL~\cite{wang2024tackling} & \underline{96.1±1.4} & \underline{0.30$\times$} & 69.7±8.0 & 0.60$\times$ & \underline{65.7±1.6} & \underline{0.24$\times$} & \underline{87.8±1.9} & 0.29$\times$ & 61.1±0.7 & \underline{0.09$\times$}\\
    \midrule
    OrthoFL (ours) & \textbf{98.2±0.3} & \textbf{0.18$\times$} & \textbf{76.5±3.3} & \textbf{0.19$\times$} & \textbf{66.2±1.3} & \textbf{0.09$\times$} & \textbf{89.7±1.0} & \textbf{0.23$\times$} & \textbf{63.0±0.6} & \textbf{0.03$\times$} \\
    \bottomrule
    \end{tabular}
    }
  \label{tab:exp-main}
\end{table*}

\begin{figure*}[t]
    \centering
    \includegraphics[width=\linewidth]{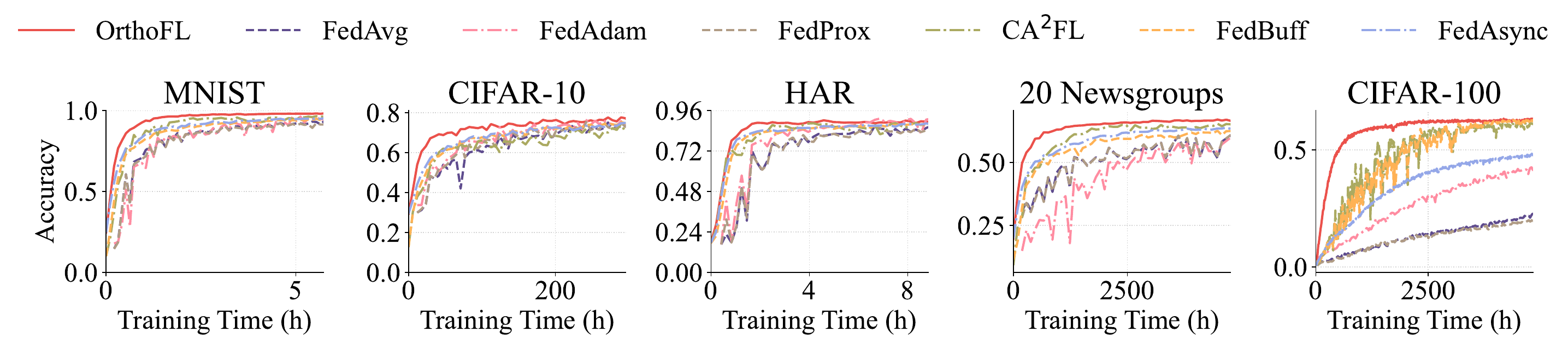}
    \caption{
    Accuracy w.r.t. training time. 
    }
\label{fig:exp-main}
\end{figure*}

\subsection{Overhead Analysis}
\begin{figure}[t]
    \centering
    \includegraphics[width=0.95\linewidth]{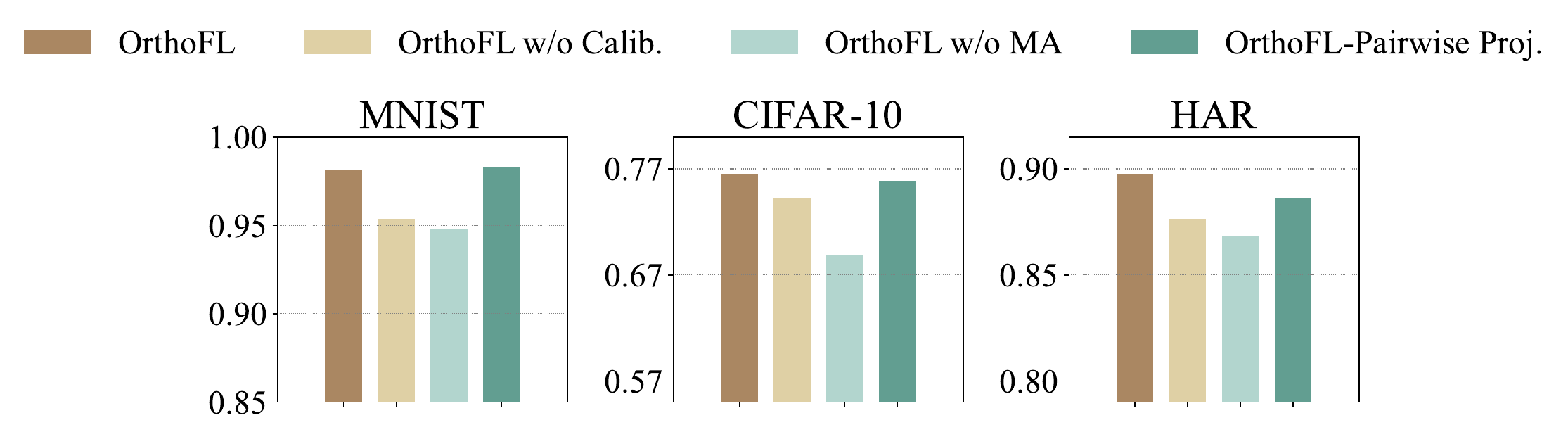}
    \caption{Ablation study. The results emphasize the importance of decoupling global and local learning to address objective inconsistencies. Besides, effective orthogonal calibration can be achieved via multiple viable approaches.
    }
\label{fig:exp-ablation}
\end{figure}
\label{sec:overhead-analysis}
Compared to baseline methods, \our does not introduce additional communication overhead. The extra computational overhead occurs on the server during the orthogonalization process.

\begin{wrapfigure}{r}{0.45\textwidth}
\renewcommand\tabcolsep{2pt}
\centering
\captionsetup{type=table}
\caption{Aggregation time of \our.}
\scalebox{0.83}{
\begin{tabular}{llcc}
    \toprule
    Dataset         & Model         & Train Params & Time        \\
    \midrule
    MNIST           & LeNet5        & 44K              & 0.003s      \\
    HAR             & ResNet18      & 119K             & 0.199s      \\
    20 Newsgroups   & DistilBERT    & 753K             & 0.110s      \\
    CIFAR-10        & VGG11         & 9.2M             & 0.188s      \\
    CIFAR-100       & MobileNetV2   & 2.4M             & 0.512s      \\
    \bottomrule
\end{tabular}
}
\label{tab:exp-aggregation-overhead}
\end{wrapfigure}
The overhead of orthogonal aggregation depends on the model size, as orthogonalization is performed through matrix operations on the weight changes of each layer. As shown in Table~\ref{tab:exp-aggregation-overhead}, the orthogonalization operation in our experiments takes between 3 ms (for MNIST) and 512 ms (for CIFAR-100) when running aggregation on a server equipped with an AMD EPYC 7713 64-Core Processor (3.72 GHz max clock) and 3.9 TiB RAM. This additional time for weight aggregation is negligible compared to the computational and communication latency on clients (shown in Figure~\ref{fig:exp-delay-simulation}).

\subsection{Ablation Studies}

We conduct ablation studies to evaluate our key design choices. First, we set FedAsync as a baseline (denoted as \our w/o Calib.), since it can be viewed as our ablation without calibration. Second, we assess the contribution of global aggregation by removing the moving average and directly loading the calibrated client weight into the global model (denoted as \our w/o MA). Furthermore, we investigate an alternative orthogonalization strategy that projects the incoming client update onto the orthogonal subspace of the most recent updates from all other clients, (denoted as \our-Pairwise Proj.). The orthogonality is achieved through the same process (i.e., removing parallel components) as in \our. 

We observe performance decreases when clients and the global model share the same weights (\our w/o Calib. and \our w/o MA). \our w/o MA performs particularly poorly as the absence of a global moving average causes overfitting to individual client distributions. These results emphasize the need to decouple global and local learning to address objective inconsistencies. Furthermore, \our-pairwise Proj. achieves performance comparable to \our, suggesting that effective orthogonal calibration can be realized through multiple viable ways. We expect that fine-tuning the orthogonalization strategy could further bring improvement to the model performance and we leave such exploration for future work. 

\begin{figure}[t]
\centering
    \includegraphics[width=0.98\linewidth]{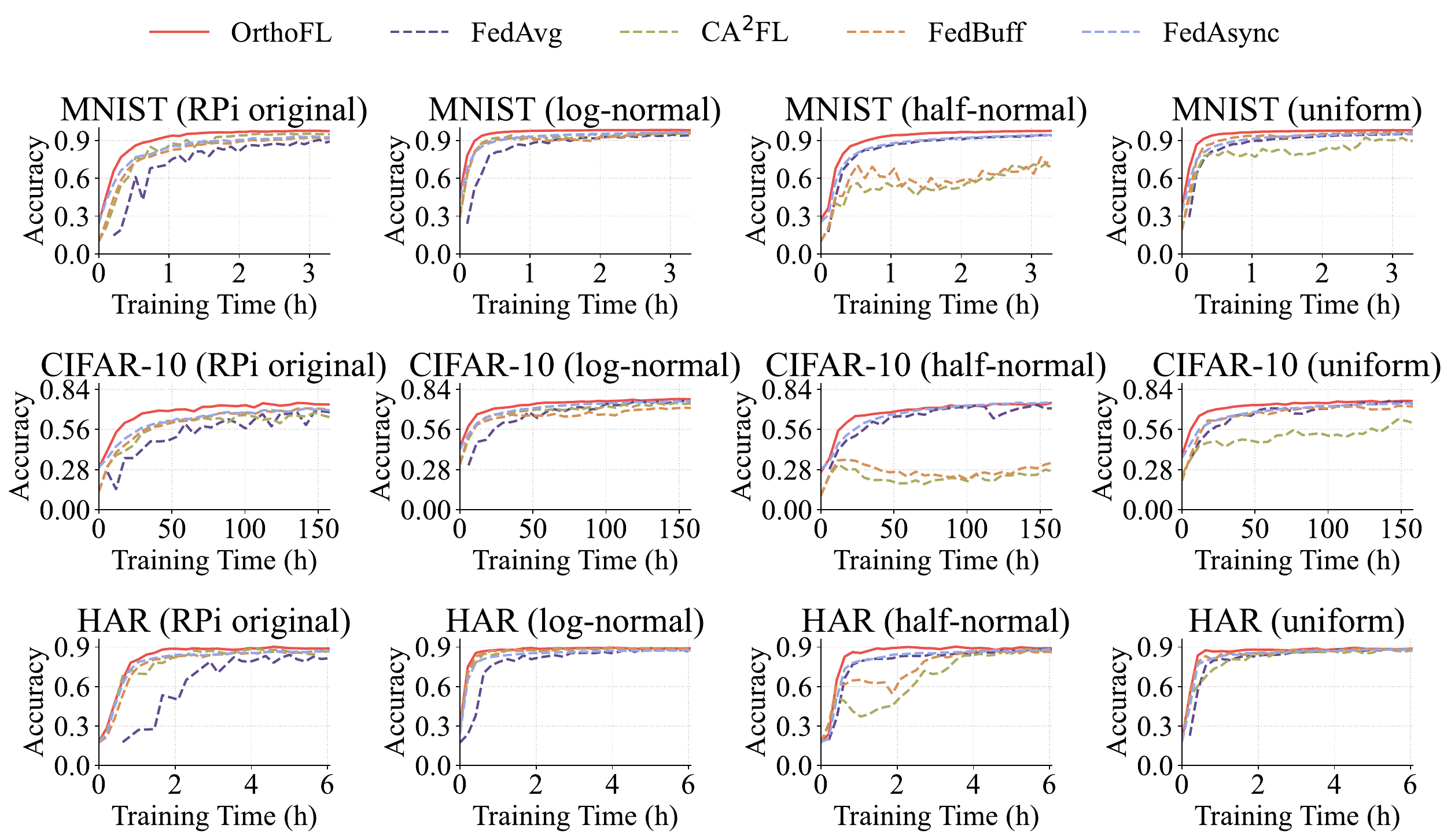}
    \caption{Performance with different delay distributions.}
\label{fig:exp-delay-dist}
\end{figure}

\subsection{Exploratory Studies}
\label{sec:exp-exploratory}
\noindent\textbf{How do algorithms perform under different delay distributions?}
Since real-world deployment of federated learning may present diverse delay patterns, we explore other possible delay distributions in real-world setups, such as following log-normal, half-normal~\cite{sui2016characterizing}, and uniform distributions~\cite{nguyen2022federated}. 

Each distribution is parameterized by the mean and variance of latency (communication and computation) observed across the RPi devices. Details on deriving these distributions are provided in Appendix~\ref{sec:derivation-delay-distribution}. Figure~\ref{fig:exp-delay-dist} presents the accuracy curves for each algorithm under different delay distributions. With real-world measured RPi latency, some clients experience substantially longer latencies, causing synchronous methods like FedAvg to converge more slowly as the server waits for stragglers. In this case, asynchronous methods generally have better performance than synchronous ones. Under the lognormal, half-normal, and uniform delay distributions, extreme latencies are less common, so the performance gap between synchronous and asynchronous methods narrows. However, the two buffer-based semi-asynchronous methods are sensitive to the delay patterns as they show lower performance under half-normal and uniform latency. In general, \our performs the best across all scenarios, demonstrating its robustness against different delay patterns.

\begin{figure}[t]
\centering
    \includegraphics[width=0.95\linewidth]{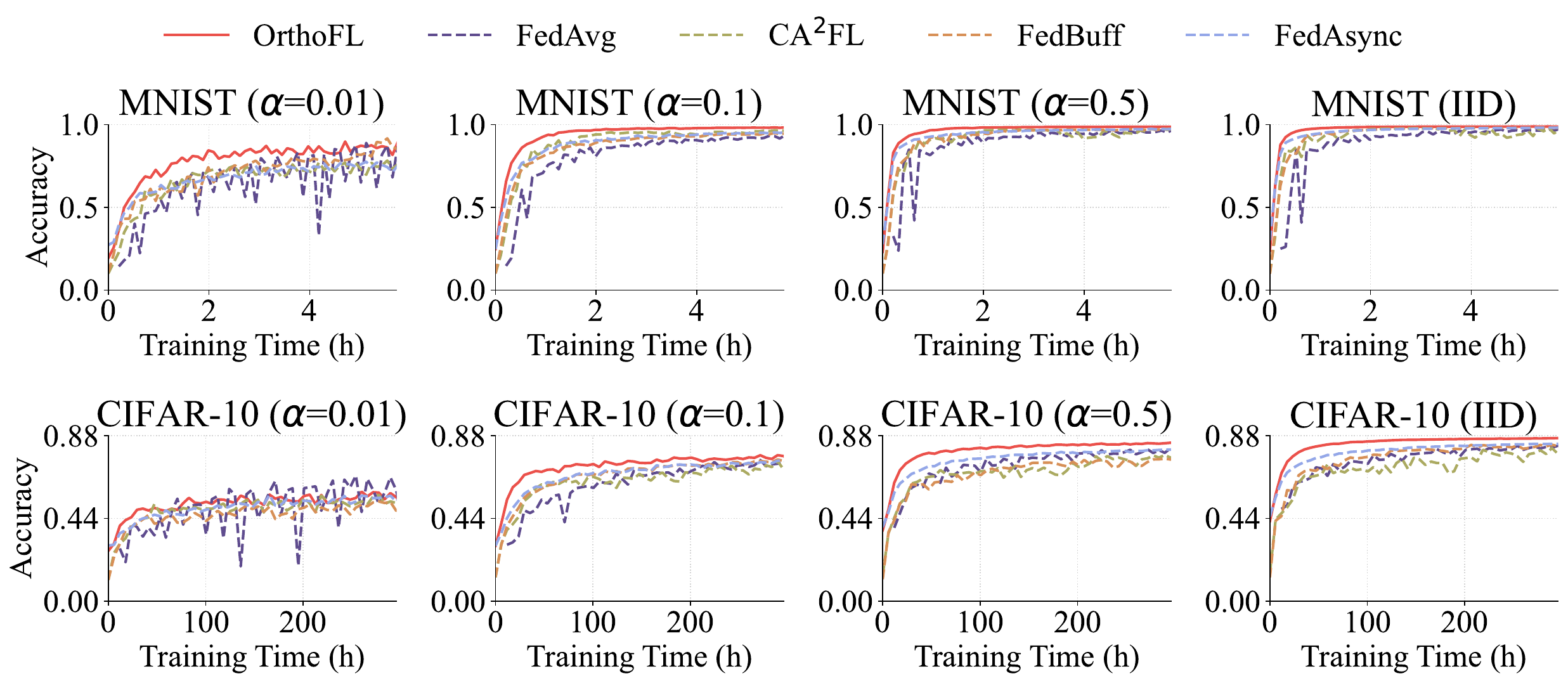}
    \caption{Performance under different data heterogeneity levels.}
\label{fig:exp-data-hetero}
\end{figure}

\noindent\textbf{How does data heterogeneity impact performance?}
To control the level of data heterogeneity, we change $\alpha$ for Dirichlet distribution from $\{0.01, 0.1, 0.5, 10^4\}$, where $\alpha = 10^4$ simulates the IID case. We present experiments on MNIST and CIFAR-10 as shown in Figure~\ref{fig:exp-data-hetero}. As the client data distribution becomes more heterogeneous (i.e., lower values of $\alpha$), we observe the performance of baseline methods has more fluctuations and decreases in final accuracy. This is because client models trained on non-IID data distributions have larger divergences in their weights. Aggregating divergent updates amplifies inconsistencies, leading to slower convergence and lower accuracy for baseline methods. By contrast, \our exhibit stable performance across all settings. In the IID case, where data is uniformly distributed across clients, \our still outperforms the compared methods. This is because our orthogonal calibration also addresses model staleness.

\subsection{Sensitivity Analyses}
\label{sec:sensitivity}
\begin{figure}[t]
    \centering
    \subfigure[Aggregation Hyperparameter $\beta$.]{
    \centering
    \includegraphics[width=0.48\linewidth]{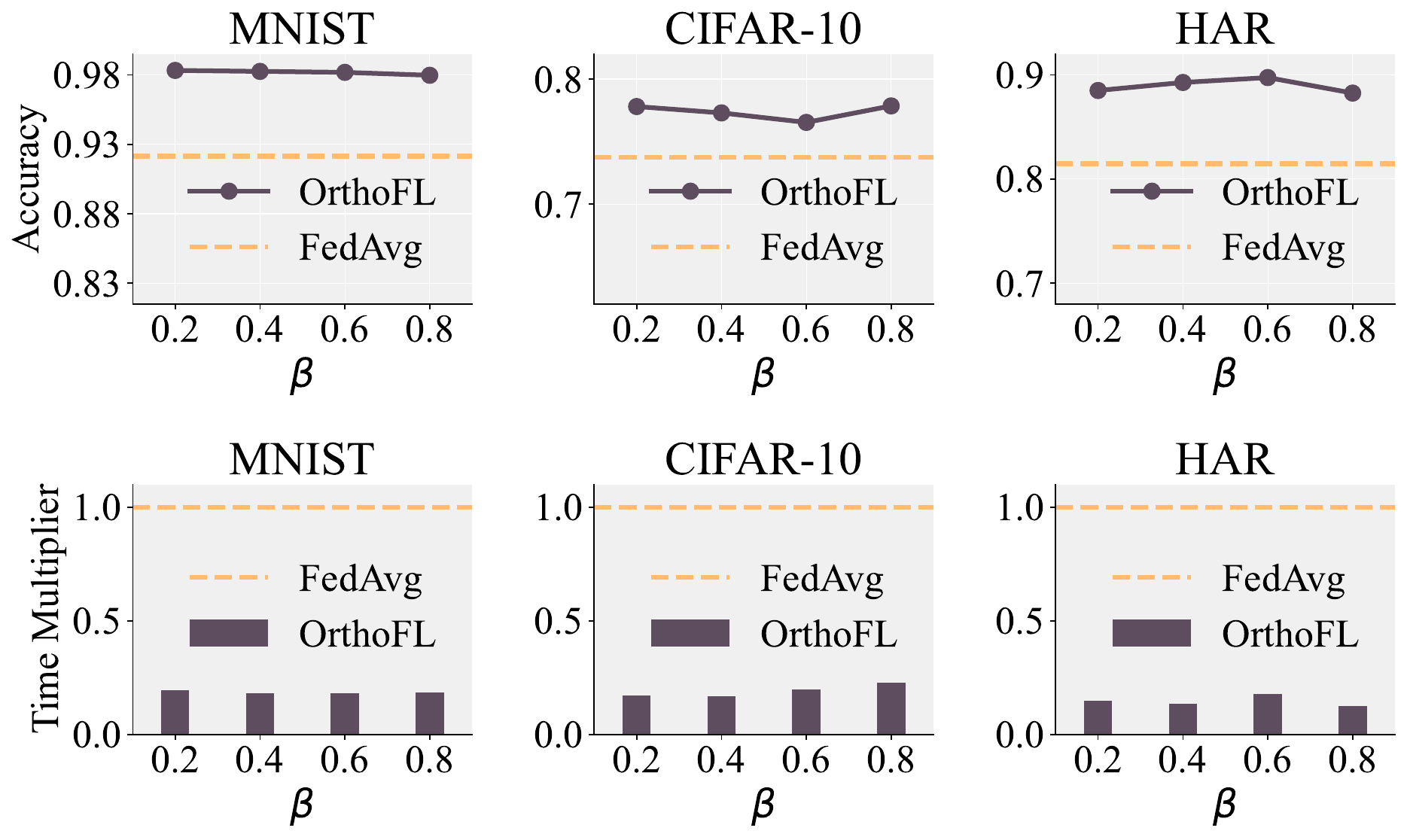}
    \label{fig:exp-sensitivity-beta}
    }
    \subfigure[Local Training Epochs $E$]{
    \centering
    \includegraphics[width=0.48\linewidth]{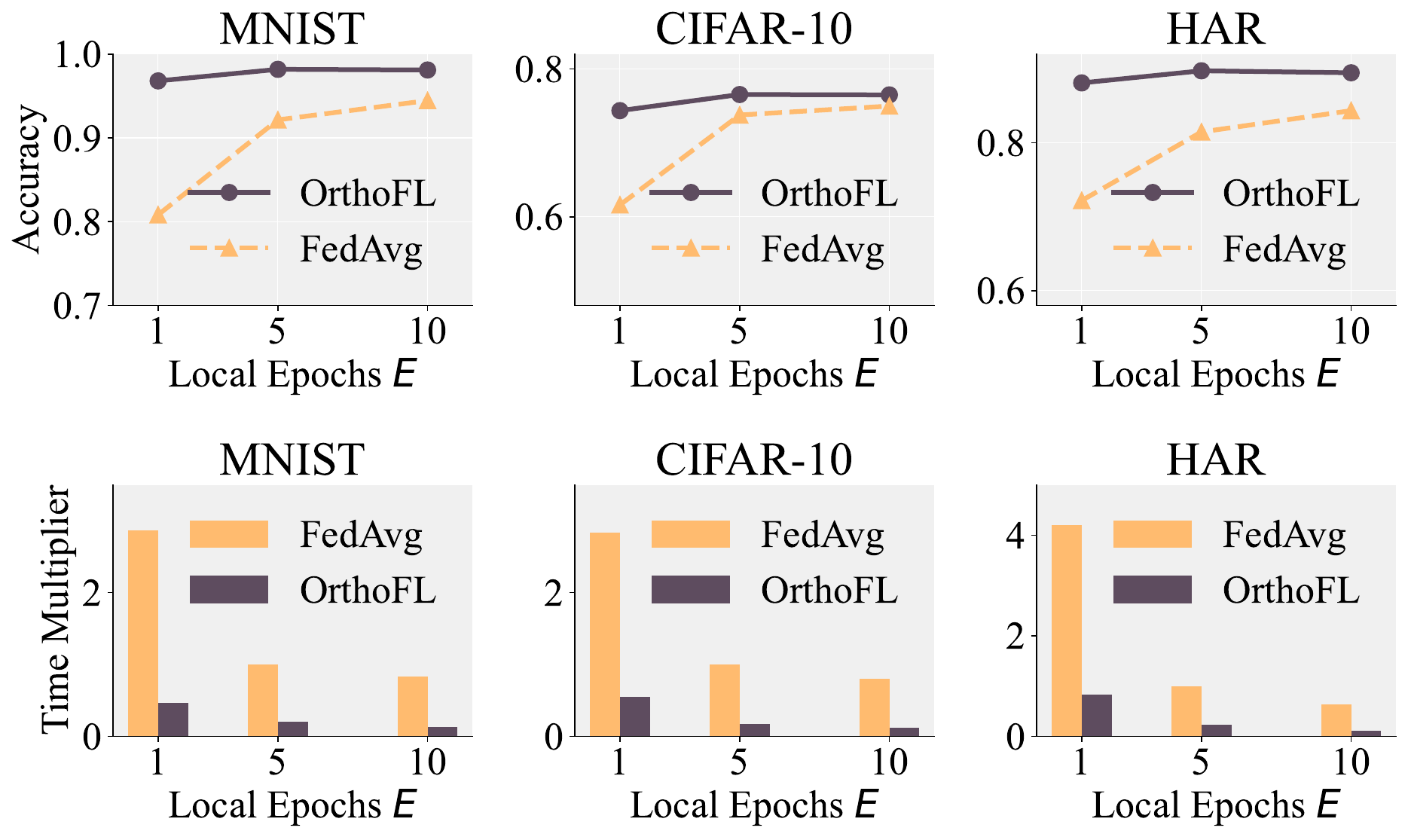}
    \label{fig:exp-sensitivity-epochs}
    }
    \caption{Sensitivity analysis. (a) \our is robust to aggregation hyperparameter $\beta$, maintaining higher accuracy than FedAvg and a low relative time. (b) An appropriately chosen number of local epochs $E$ improves accuracy within the same training duration and expedites convergence.}
\end{figure}

\smallsection{Aggregation Hyperparameter} 
The parameter $\beta$ in Equation~\ref{eq:aggregation} acts as a smoothing factor that balances the contribution of client updates to the global model and the retention of the global model's weights from the previous round. A larger $\beta$ allows the global model to incorporate more of the client updates, potentially accelerating learning, while a smaller $\beta$ preserves more of the global model's previous state, enhancing stability. We vary the value of $\beta$ from $\{0.2, 0.4, 0.6, 0.8\}$. As shown in Figure~\ref{fig:exp-sensitivity-beta}, \our is robust to different $\beta$ values and always achieves higher accuracy than FedAvg after a fixed training time, and the relative time of \our is low.

\smallsection{Number of Local Training Epochs} The number of local training epochs $E$ determines how much information is learned during a round and impacts the overall convergence speed. A larger $E$ allows clients to learn more information from their local data, potentially improving local model performance. However, it has the risk of larger divergence among client models and global models. On the other hand, setting a smaller $E$ ensures closer alignment between client and global models but comes at the cost of higher communication latency due to more frequent synchronization. We vary the value of $E$ from $\{1, 5, 10\}$ and show the results in Figure~\ref{fig:exp-sensitivity-epochs}. The upper row shows the final accuracy after a fixed training time and the bottom row presents the relative time to reach 95\% of FedAvg's target accuracy when $E=5$. We observe that when $E=1$, both methods reach lower accuracy compared to larger $E$. This is due to insufficient local learning, requiring more communication rounds to achieve comparable performance. \our achieves similar final accuracy when $E=5$ and $E=10$. Notably, \our outperforms FedAvg across different $E$ values. The speedup of \our is more obvious at smaller $E$ values (e.g., $E=1$).
\section{Related Works}
We review two core directions relevant to our study. Connections with other areas including asynchronous stochastic gradient descent and continual learning are discussed in Appendix~\ref{sec:connections-with-other-areas}.

\smallsection{Federated Learning and Heterogeneity Problem}
Federated learning~\cite{mcmahan2017communication} is a distributed learning paradigm that allows multiple parties to jointly train machine learning models without data sharing, preserving data privacy. Despite the potential, it faces significant challenges due to heterogeneity among participating clients, which is typically classified into two main categories: data heterogeneity and system heterogeneity. Data heterogeneity appears as clients own non-IID (independent and identically distributed) data~\cite{li2020federated,karimireddy2020scaffold,wang2020tackling,zhang2023navigating}. The difference in data distribution causes the local updates to deviate from the global objective, making the aggregation of these models drift from the global optimum and deteriorating convergence. System heterogeneity refers to variations in client device capabilities, such as computational power, network bandwidth, and availability~\cite{wang2020tackling,zhang2021parameterized,li2021fedmask,fang2022robust,alam2022fedrolex,zhang2024few}. These disparities lead to uneven progress among clients, and the overall training process is delayed by slow devices. Traditional federated learning approaches rely on synchronization for weight aggregation~\cite{mcmahan2017communication,li2020federated,reddi2020adaptive}, where the server waits for all clients selected in a round to complete and return model updates before proceeding with aggregation. This synchronization leads to inefficient resource utilization and extended training times, particularly in large-scale deployments involving hundreds or thousands of clients. Addressing the heterogeneity issues is a critical problem for improving the scalability and efficiency of federated learning systems in real-world deployment.

\smallsection{Asynchronous Federated Learning}
Much of the asynchronous federated learning literature focuses on staleness management by assigning weights for aggregating updates according to factors including delay in updates~\cite{xie2019asynchronous}, divergence from the global model~\cite{su2022asynchronous,zang2024efficient} and local losses~\cite{liu2024fedasmu}. For example, \cite{xie2019asynchronous} lets the server aggregate client updates into the global model with a weight determined by staleness. Another line of research caches client updates at the server and reuses them to calibrate global updates~\cite{gu2021fast,wang2024tackling}. For example, \cite{wang2024tackling} maintains the latest update for every client to estimate their contribution to the current aggregation and calibrate global updates. Furthermore, semi-asynchronous methods~\cite{nguyen2022federated,zang2024efficient} balance between efficiency and training stability. For example, \cite{nguyen2022federated} buffers a fixed number of client updates before aggregation. We select representative methods from each category for our comparisons. Besides, some works improve efficiency from a different perspective---through enhanced parallelization. Methods include decoupling local computation and communication~\cite{avdiukhin2021federated} and parallelizing server-client computation~\cite{zhang2023no}. In addition, asynchronous architectures have been explored in other paradigms such as vertical~\cite{zhang2024asynchronous} and clustered~\cite{liu2024casa} federated learning. While these directions complement our work, they fall outside the scope of this study.
\section{Conclusions}
In this paper, we introduce \our, an orthogonal calibration mechanism for asynchronous federated learning. \our exploits the high-dimensional parameter space of neural networks and projects the global weight shift during a client's delay onto a subspace orthogonal to its stale update. This projection ensures global progress is integrated into client models without disrupting local learning. Experiments demonstrate that \our consistently outperforms state-of-the-art synchronous and asynchronous baselines, achieving notable gains in accuracy, convergence speed, and robustness under diverse delay patterns and data heterogeneity.

For future work, we plan to explore more complicated scenarios such as dynamic client participation and failure tolerance. Discussions on the adaptability of \our to these challenges are provided in Appendix~\ref{sec:discussion}. We also aim to further optimize learning efficiency, potentially through integrating adaptive client selection into \our to prioritize impactful clients for participation. These explorations will further enhance the robustness and scalability of federated learning. 

\bibliographystyle{ACM-Reference-Format}
\bibliography{cited}

\newpage
\clearpage
\appendix
\section{Proof}
\label{sec:proof}
In this section, we present the proof for Lemma~\ref{eq:lemma}. We restate the lemma below:
\begin{lemma}
    Let $v \in \mathbb R^d$ and $\mathcal{U} = \{u_1, \cdots, u_k\}$ be an orthonormal set for some $k < d$. Then for any $w \in (\mathrm{span}\,\mathcal{U})^\perp$, 
    \begin{equation}
        \label{eq:appendix-lemma}
        \Vert v - v^\perp \Vert \le \Vert v - w \Vert,
    \end{equation}
    where $v^\perp:= v - \sum_{i=1}^k \langle v, u_i\rangle u_i$ denote the component of $v$ orthogonal to $\mathcal{U}$. Moreover, the angle between $v$ and $v^\perp$ is less than any angle between $v$ and $w$ for $w \in (\mathrm{span}\,\mathcal{U})^\perp$.
\end{lemma}

\begin{proof}
    Extend $\mathcal{U}$ to an orthonormal basis $\mathcal{B} := \{u_1, \cdots, u_k, u_{k+1}, \cdots, u_d\}$ on $\mathbb R^d$. Write $w = \sum_{i=1}^d w_i u_i$ and $v = \sum_{i=1}^d v_i u_i$, where $w_i:= \langle w, u_i\rangle$ and $v_i := \langle v, u_i \rangle$ denote the coordinates of $w$ and $v$ with respect to the basis $\mathcal{B}$ respectively. Since $w \in (\mathrm{span}\,\mathcal{U})^\perp$, $w_i = 0$ for $1\le i \le k$. It follows from Pythagorean theorem that
    \begin{equation}
    \begin{aligned}
        \Vert v - w \Vert^2 &= \left\Vert \sum_{i=1}^d (v_i - w_i) u_i \right\Vert^2 \\
        &= \left\Vert \sum_{i=1}^k v_i u_i + \sum_{j=k+1}^d (v_j - w_j) u_j \right\Vert^2 \\
        &= \left\Vert \sum_{i=1}^k v_i u_i \right\Vert^2 + \left\Vert\sum_{j=k+1}^d (v_j - w_j) u_j \right\Vert^2 \\
        &\ge \left\Vert \sum_{i=1}^k v_i u_i \right\Vert^2.
    \end{aligned}
    \end{equation}
    Taking square root of both sides and noting that by definition $v - v^\perp = \sum_{i=1}^k v_i u_i$, equation~\eqref{eq:appendix-lemma} follows.
    As for the second claim, let $\theta(v,w)$ denote the angle between $v$ and $w$, $\theta(v, v^\perp)$ the angle between $v$ and $v^\perp$. 
    By the Cauchy-Schwarz inequality,
    \begin{equation}
        \begin{aligned}
            \langle v ,w \rangle &= \left\langle v, \sum_{j = k+1}^d w_j u_j\right\rangle \\
            &= \sum_{j=k+1}^d w_j  \langle v, u_j \rangle \\
            &\le \left(\sum_{j=k+1}^d w_j^2\right)^{1/2} \cdot \left(\sum_{j=k+1}^d v_j^2 \right)^{1/2}\\
            &= \Vert w \Vert \cdot \Vert v^\perp \Vert.
        \end{aligned}
    \end{equation}
    It follows that 
    \begin{equation}
        \cos \theta(v,w) := \frac{\langle v, w\rangle}{\Vert v\Vert \, \Vert w\Vert} \le \frac{\Vert v^\perp\Vert}{\Vert v\Vert} = \frac{\langle v, v^\perp \rangle}{\Vert v\Vert\, \Vert v^\perp\Vert} = \cos \theta (v, v^\perp).
    \end{equation}
    Since the consine function is monotonically decreasing on $[0, \pi]$, $\theta(v,w) \ge \theta(v, v^\perp)$ as claimed.
\end{proof}

\section{Details of Experiment Setup}
\subsection{Applications and Datasets}
\label{sec:application-and-dataset}
The experiments are conducted with the following three applications. 
\begin{enumerate}[leftmargin=*]
    \item \textbf{Image Classification.} We evaluate our framework on three widely used image datasets: MNIST~\cite{deng2012mnist}, CIFAR-10, and CIFAR-100~\cite{krizhevsky2009learning}. For MNIST, we use LeNet5~\cite{lecun1998gradient}, a lightweight convolutional network. For CIFAR-10, we adopt VGG11~\cite{simonyan2014very}, a deeper convolutional architecture. For CIFAR-100, we employ MobileNetV2~\cite{sandler2018mobilenetv2}, a compact and efficient model ideal for large-scale image classification tasks.
    \item \textbf{Text Classification.} We experiment with the 20 Newsgroups dataset~\cite{lang1995newsweeder}, a benchmark dataset for multi-class text categorization. We adopt DistillBERT~\cite{sanh2019distilbert}, a small transformer model suitable for resource-constrained devices. 
    \item \textbf{Human Activity Recognition.} We use the HAR~\cite{anguita2013public} dataset, which contains time-series sensor data for different physical activities. We adopt the 1D version of ResNet18~\cite{he2016deep}, a modified ResNet architecture for processing 1D sequential data.
\end{enumerate}

\subsection{Baseline Methods}
\label{sec:baseline-methods}
Below are the baseline methods that we compared in the experiments:
\begin{itemize}[leftmargin=*]
    \item \textbf{FedAvg}~\cite{mcmahan2017communication} is the classical synchronous algorithm where the server selects a subset of clients to conduct training in each round and synchronizes updates from these clients before aggregation.
    \item \textbf{FedProx}~\cite{li2020federated} is a synchronous method that addresses data heterogeneity by incorporating L2 regularization during local training to constrain the divergence between global and client models.
    \item \textbf{FedAdam}~\cite{reddi2020adaptive} is a synchronous algorithm that integrates Adam optimizer for the server. It adapts learning rates for each parameter using first- and second-moment estimates, improving convergence and accelerating training under data heterogeneity.
    \item \textbf{FedAsync}~\cite{xie2019asynchronous} is a fully-asynchronous framework that lets the server immediately aggregate the client updates into the global model upon receipt. It uses a weighting mechanism as in Equation~\ref{eq:fedasync} to account for the staleness of the updates. 
    \item \textbf{FedBuff}~\cite{nguyen2022federated} is a semi-asynchronous framework that introduces a buffered aggregation strategy. It maintains a buffer to collect client updates. Once the buffer is full, the server aggregates the client updates in the buffer and updates the global model.
    \item \textbf{CA$^2$FL}~\cite{wang2024tackling} is another semi-asynchronous framework based on buffered aggregation. It caches the latest update from every client on the server and uses them to estimate the clients' contribution to the update of the current round and calibrate global updates.
\end{itemize}

\subsection{Federated Learning Configuration}
For reproducibility, we report the configurations in our experiments. 
The number of local training epochs $E=5$. For the FedAvg algorithm, the number of sampled clients at each round is 10. The aggregation hyperparameters in Equation~\ref{eq:aggregation} are set to $\beta = 0.6$ and $a = 0.5$, following the values used in prior work~\cite{xie2019asynchronous}. The learning rate for local training at all clients is $5 \times 10^{-5}$ for the 20 Newsgroups dataset and 0.01 for the other datasets. 

\subsection{Derivation of Delay Distributions}
\label{sec:derivation-delay-distribution}

Figure~\ref{fig:exp-delay-simulation} presents the average communication and computation latency on Raspberry Pis scaled for the dataset and model configuration in our experiments. For 20 Newsgroups, the converted computational time is sufficiently long, making communication time negligible.

\begin{figure}[htbp]
\centering
    \includegraphics[width=0.8\linewidth]{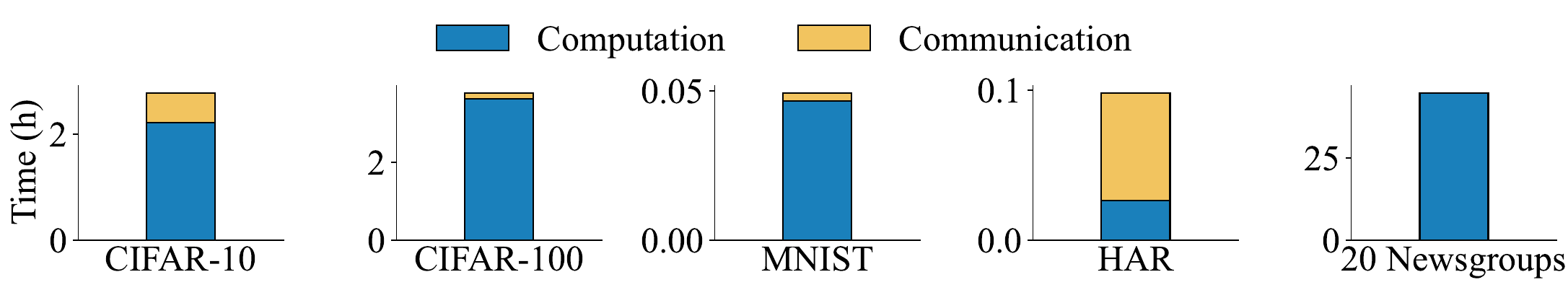}
    \caption{Average latency per round across clients.}
\label{fig:exp-delay-simulation}
\end{figure}

\smallsection{Additional Delay Distributions} In Section~\ref{sec:exp-exploratory}, we present the performance of compared methods under additional delay distributions following physical deployments. The derivations of the additional delay distributions are as follows.
\begin{itemize}[leftmargin=*]
    \item \textbf{Lognormal distribution}: The parameters $\mu$ (mean) and $\sigma$ (standard deviation) of the natural logarithm of delays are derived from the measurements on Raspberry Pis. Specifically, $\mu_R$ and  $\sigma_R$ represent the arithmetic mean and standard deviation of the measured delays for all rounds, respectively. 
        $$
        \sigma = \sqrt{\ln\left(\frac{\sigma_R^2}{\mu_R^2} + 1\right)},\quad
        \mu = \ln(\mu_R) - \frac{\sigma^2}{2}
        $$
    Then, the latency for each client is sampled from the lognormal distribution to capture skewed and heavy-tailed delays. 
    \item \textbf{Half-normal distribution}: The mean and standard deviation of the delays (i.e., $\mu_R$ and  $\sigma_R$) are calculated from the delay measurements on Raspberry Pis. Then, for each client, its latency is sampled from the half-normal distribution, ensuring non-negative values and a skewed distribution toward smaller delays.
    \item \textbf{Uniform distribution}: Client latency is sampled from a uniform distribution with bounds set between the 5th and 95th percentiles of the measurements from Raspberry Pis. This ensures outliers are excluded while covering the majority of the observed range.
\end{itemize}

\begin{figure}[htbp]
\centering
    \includegraphics[width=0.8\linewidth]{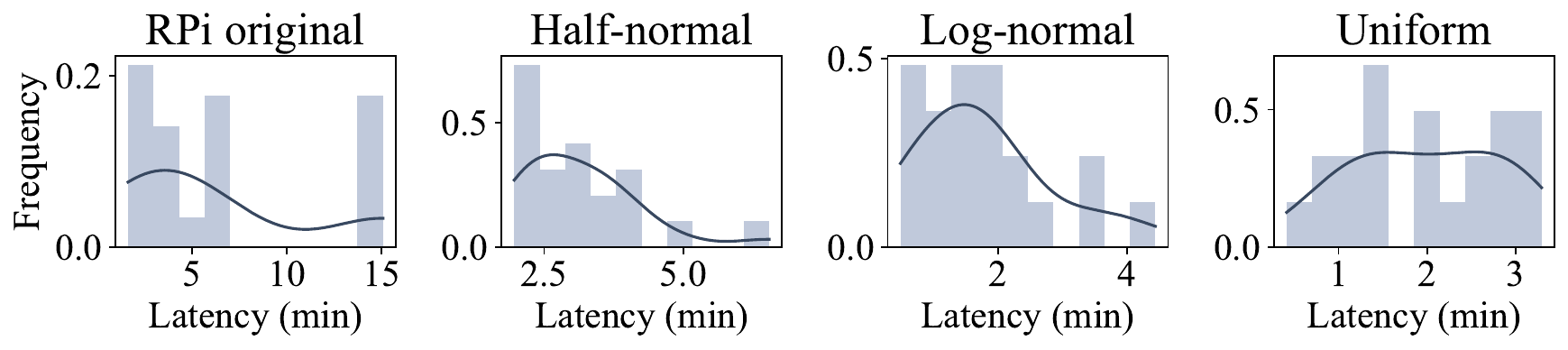}
    \caption{Delay patterns under different distributions: measurements from RPis show discrete peaks and high variability, half-normal and lognormal distributions have long tails, and uniform distribution assumes equal probability within a bounded range.
    }
\label{fig:exp-diff-delays}
\end{figure}

Figure~\ref{fig:exp-diff-delays} shows the simulated latency for 100 clients, each running CIFAR-100 with MobileNetV2 over 10,000 rounds. 

\section{Connections with Other Areas}
\label{sec:connections-with-other-areas}
\smallsection{Asynchronous Stochastic Gradient Descent}
Asynchronous stochastic gradient descent (SGD) is closely related to asynchronous federated learning and has provided theoretical and empirical foundations for scalable distributed training. Early studies analyzed error-runtime trade-offs, showing that incorporating stale gradients can alleviate system bottlenecks without significantly compromising accuracy~\cite{dutta2018slow}. Subsequent work refined convergence bounds based on maximum~\cite{stich2019error} or average~\cite{koloskova2022sharper} delay and demonstrated that asynchronous SGD can converge faster than traditional minibatch SGD~\cite{mishchenko2022asynchronous}. To tackle challenges such as gradient staleness, communication delays, and convergence guarantees, various strategies have been proposed, such as filtering out outlier gradients~\cite{xie2020zeno++,cohen2021asynchronous}, adjusting update steps according to delay~\cite{mishchenko2018delay,aviv2021learning}, and approximating gradients to compensate for delayed information~\cite{zheng2017asynchronous}. However, unlike federated learning, most asynchronous SGD formulations do not explicitly address non-i.i.d. data distributions or the strict data privacy constraints inherent in federated settings, which limits their direct applicability. 

\smallsection{Continual Learning} There are works in continual learning~\cite{farajtabar2020orthogonal,chaudhry2020continual,saha2021gradient,lin2022trgp,chen2022class} that leverage the idea of orthogonalization to project gradients onto non-conflicting subspaces across tasks. Continual learning aims to prevent catastrophic forgetting, where new knowledge overrides previously learned information when training a model on a sequence of tasks. It shares a common goal with asynchronous federated learning, which is to mitigate knowledge interference. Despite the similarities, the two fields follow distinct training paradigms. Continual learning typically processes tasks sequentially in a centralized setting. Methods can reuse data or access sample-level information (e.g., gradient) from previous tasks. In contrast, federated learning involves repeated rounds of training from distributed clients with data constraints. The model must integrate updates without knowing data or gradients from clients to preserve privacy.
Therefore, the two fields require different optimization techniques and system designs.

\section{Adaptability to Dynamic Environments}
\label{sec:discussion}
The asynchrony and orthogonal calibration mechanisms not only accelerate training but also extend its applicability to more complex scenarios. These features make \our suitable in dynamic federated learning environments. We discuss the following situations:

\smallsection{Dynamic Client Participation}
In real-world federated learning applications, it is common for new clients to join the training process~\cite{park2021tackling} or for previously unseen data with new tasks to appear over time~\cite{yang2024federated,yoon2021federated,ma2022continual,dong2022federated}. Scalability to new clients and tasks is essential for long-term performance. \our is readily applicable to these settings. By orthogonal calibration on asynchronous updates, \our mitigates the disruptive impact of newly joined clients, especially if the global model is well-trained and initial updates from newly joined clients are noisy or biased due to limited data or distribution shifts. Similarly, when new classes are introduced, \our preserves knowledge of previously learned classes, reducing forgetting and performance degradation caused by sudden distribution shifts.

\smallsection{Failure Handling} Large-scale federated learning faces a high risk of client or system failures due to issues such as out-of-memory (OOM) errors during training, hardware malfunctions, or network interruptions~\cite{jang2023oobleck,gupta2017failures,jeon2019analysis,weng2022mlaas}, particularly when training large models across many devices. Failure rates tend to increase with larger models and more clients, as the computational and communication demands scale up significantly. The ability to effectively handle these failures is critical for ensuring robust and reliable training in federated learning systems. The asynchronous nature of \our makes it robust to such failures. When a client fails, training continues with updates from the remaining clients. A single node's failure does not block global progress.

\section{Pseudo-code of \our}
\label{sec:pseudo-code}
\begin{algorithm}[htbp]
    \SetKwData{Left}{left}\SetKwData{This}{this}\SetKwData{Up}{up}
    \SetKwFunction{Union}{Union}\SetKwFunction{FindCompress}{FindCompress}
    \SetKwInOut{Input}{Input}\SetKwInOut{Output}{Output}
    
    \Input{Total number of updates $T$, local training epochs $E$, initial global model weights $W^{(0)}$. 
    }
    \Output{Global model weights $W^{(T_+)}$.}
    
    \SetAlgoNlRelativeSize{0}
    
    \underline{\textbf{Server execution:}} \\
    Send $W^{(0)}$ to all available clients\;
    \For{$t = 1, \dots, T$}{
            Receive update from a client\;
            Calculate delay $t - \tau$\;
                Compute global weight shift: $\Delta W = W^{(t)} - W^{(\tau_+)}$ \;
                Compute client weight change: $\Delta W_m = W_m^{(t)} - W_m^{(\tau_+)}$ \;
                Orthogonalize each layer: $\Delta W^{l\perp} = \Delta W^{l} - \frac{\Delta W^{l} \cdot \Delta W^{l}_m}{\Delta W^{l}_m \cdot \Delta W^{l}_m} \Delta W^{l}_m$\;
                Merge weights for client: $W_m^{(t_+)} = W_m^{(t)} + \Delta W^\perp$\;
            Update global model: $W^{(t_+)} = (1 - \beta_t) W^{(t)} + \beta_t W_m^{(t)}$\;
            ClientUpdate($m$, $W_m^{(t_+)}$)
    }
    \Return $W^{(T_+)}$\;
    
    \underline{\textbf{ClientUpdate}}($m$, $\tilde{W}$):\\
    Initialize local model: $W_{m} \leftarrow \tilde{W}$\;
    \For{$e = 1, \dots, E$}{
        Partition $\mathbf{D}_m$ into mini-batches $\bigcup_{i=1}^{j_m} B_i^{(m)}$\; 
        \For{$i = 1, \dots, j_m$}{
            Update local weights: $W_m \leftarrow W_m - \eta_c\nabla_{W_{m}}{\mathcal{L}}_{m}(W_m; B^{(m)}_i)$ \; 
        }
    }
    \Return $W_{m}$ to server\;
  
  \caption{\our Framework}
  \label{algo:federated-framework}
\end{algorithm}

\end{document}